\documentclass[1p,times,number]{elsarticle}
\usepackage{amsmath,amssymb,amsthm}
\usepackage{graphicx}
\usepackage{booktabs}
\usepackage{multirow}
\usepackage{array}
\usepackage{xcolor}
\usepackage{enumitem}
\usepackage{algorithm}
\usepackage{microtype}
\usepackage{placeins}
\usepackage{tikz}
\usetikzlibrary{arrows.meta, positioning}
\usepackage{soul}
\newtheorem{proposition}{Proposition}

\newtheorem{remark}{Remark}
\journal{Applied Soft Computing}
\begin{document}
\begin{frontmatter}

\title{Logit-Origin Centering for Singleton Test-Time Adaptation}

\author[inst1]{Mayank Sharma}
\author[inst1]{Rohit Kumar Mourya}
\author[inst1]{Pratik Mazumder}
\affiliation[inst1]{organization={Indian Institute of Technology Jodhpur},
            city={Jodhpur}, country={India}}

\begin{abstract}

Tabular data is used extensively in many real-world use cases. Deep learning models have been developed to deal with tabular data, but generally perform poorly when the test data distribution differs from that of the training data. Researchers have proposed test-time adaptation approaches to deal with this problem. The fully test-time adaptation (FTTA) setting involves adapting deployed classifiers to shifted target distributions using only unlabeled test data. Leading FTTA methods inherit a batch-dependent approach from computer vision literature. This paper demonstrates for the first time that such approaches degrade sharply in strict streaming regimes where examples arrive and must be classified one at a time. This occurs because at a batch size of one, batch-level statistics become unavailable or poorly estimated. We argue that singleton tabular FTTA is not merely a small-batch variant of ordinary FTTA, but a distinct identifiability problem where only the location of the model's score stream remains directly observable. To address this, we propose \emph{Prequential Logit-Origin Centering (PLOC)}, a lightweight approach that keeps the source model frozen and shifts the logit space at each step. PLOC stores only a single running number (the mean of past logits), requires no labels, estimates no priors, and bypasses weight updates entirely. A deferred variant applies a static shift that preserves the source ranking, and thus the AUROC, exactly. Evaluated across five tabular benchmarks, three architectures (MLP, FT-Transformer, and TabTransformer), and five independent source checkpoints, PLOC significantly outperforms strong tabular and entropy-based baselines. 

\end{abstract}

\begin{highlights}
\item Existing test-time adaptation methods are unstable or undefined for a unit-size test batch.
\item We propose a tuning-free test-time adaptation method PLOC for tabular data.
\item PLOC significantly outperforms existing methods on unit size test batches.

\end{highlights}


\end{frontmatter}

\section{Introduction}
A classifier trained on a source distribution is rarely deployed on data drawn from the same distribution. For example, in credit scoring, the applicant pool shifts with the economy; in clinical risk, the case mix differs between hospitals; in education, the student population changes from term to term. Fully test-time adaptation (FTTA) addresses this gap by allowing the model, or its outputs, to adjust using only unlabeled target data, without revisiting the source set or collecting target labels. For tabular data, where gradient-boosted trees and increasingly deep networks are deployed in settings where distribution shift is common, FTTA is an attractive way to recover accuracy lost to drift.

Most existing tabular FTTA methods, however, inherit a design assumption from the vision literature that test time data arrive in \emph{batches} large enough to estimate distributional statistics. FTAT~\cite{ftat} is a recent approach that forms a label-distribution objective and updates model parameters. Its robustness study uses batches of size between $64$ and $1024$. PFT3A~\cite{pft3a} estimates class priors, learns robust features, and explores representative subspaces. The standard adaptation toolkit (Tent~\cite{tent}, EATA~\cite{eata}, SAR~\cite{sar}, and LAME~\cite{lame}) minimizes batch entropy or refines batch outputs through an affinity graph. All of these methods are batch estimators and rely on batch-level statistics.

In this work, we study the strict singleton setting: \textbf{batch size one}. Each prediction must be produced before the next sample is seen, and only a single test point is available at each step. This regime arises in streaming decision systems, where one transaction is approved or declined, one patient is triaged, or one student response is scored before the next arrives, and it exposes a limitation of batched estimators. With a single point, a label-prior estimate is not reliable, an entropy update is driven by one prediction, a neighborhood graph is a single node, and methods that adapt batch-normalization statistics have no batch over which to compute them.

We treat singleton tabular FTTA as an identifiability problem rather than as a small-batch variant of standard FTTA. When only one sample is available per step, batch-level target quantities are difficult to identify, but the score location can still be estimated from the stream. When the source ranking remains useful,  but the score origin shifts, the running mean of target logits provides an unlabeled estimate of that shift. We propose \textbf{Prequential Logit-Origin Centering (PLOC)} that addresses the singleton tabular FTTA problem. At each step, the proposed approach subtracts the running mean of previously observed logits before the output activation and emits the result, with the rule and its analysis given in Section~\ref{sec:method}. PLOC stores only one running number, uses no target labels, estimates no class prior, and updates no model parameters. This gives a minimal adaptation rule for the singleton setting.

The major contributions of this work are as follows:
\begin{itemize}
\item \textbf{Method.} PLOC and its deferred variant: to the best of our knowledge, this is the first singleton-native tabular FTTA method, which adapts only the prequential output-logit origin of a frozen classifier (Section~\ref{sec:method}).
\item \textbf{Theory.} We prove that deferred PLOC preserves source AUROC \emph{exactly} (Proposition~\ref{prop:auroc}) and that PLOC's ranking drift is bounded by a pair-margin quantity that vanishes under running-mean convergence and a no-mass-near-zero-margin condition (Proposition~\ref{prop:drift}). The exact guarantee is confirmed numerically in every dataset--backbone--checkpoint combination (Section~\ref{sec:theory}).
\item \textbf{Evidence.} Across five tabular benchmarks, three backbones, and five source checkpoints at batch size one, PLOC improves accuracy, balanced accuracy, F$_1$, calibration, and likelihood over the tabular method FTAT while leaving AUROC essentially unchanged. The prior-free method PFT3A is undefined at $B{=}1$, and the entropy- and normalization-based methods (Tent, EATA, SAR, and LAME) leave the source predictions unchanged, or nearly so, because the batch-level quantities that drive their updates degenerate when only a single sample is available. A batch-size sweep from one to $1024$ shows PLOC is batch-invariant while batch methods recover only with a larger batch (Sections~\ref{sec:results}--\ref{sec:batch}).
\item \textbf{Scope and failure modes.} We map where PLOC helps and where it does not under controlled prior shift (Section~\ref{sec:scope}).
\end{itemize}

\section{Related Works}
\paragraph{Tabular fully test-time adaptation} FTAT~\cite{ftat} addresses tabular FTTA through a Confident Distribution Optimizer, a Local Consistency Weighter, and Dynamic Model Ensembling, all computed on the current test batch; its robustness study operates with batches between $64$ and $1024$. PFT3A~\cite{pft3a} targets prior-free tabular adaptation by estimating class priors, learning robust features, and exploring representative subspaces. Both are strong methods in their intended batched setting, and both are batch estimators: their objectives are defined over a population of test points rather than a single one. We show in Section~\ref{sec:results} that their core objects degenerate at batch size one, and we verify in Section~\ref{sec:batch} that PFT3A recovers its published batched performance once a larger batch is available, so its singleton failure is a property of the regime and not of our integration. PLOC is complementary to this line: rather than adapting features or priors, it corrects only the scalar location of the output stream, a quantity that remains observable from the stream. More broadly, transfer- and domain-adaptation methods handle distribution shift at training time, through deep transfer with limited target data~\cite{asoc_transfer}, instance-based transfer using domain-adversarial target-data generation and influence-function selection~\cite{asoc_dann}, and multi-source-to-multi-target alignment~\cite{asoc_m2m}. Unlike FTTA, these methods require access to target data before deployment.

\paragraph{Test-time adaptation in vision and the batched assumption} The modern TTA literature originates in image classification, where a model is adapted to corrupted or shifted test data without labels. Tent~\cite{tent} minimizes prediction entropy over batch-normalization affine parameters; EATA~\cite{eata} adds reliability and diversity filtering to choose which test samples drive the update; SAR~\cite{sar} stabilizes the entropy objective with sharpness-aware minimization over normalization parameters. Continual and temporally correlated variants such as CoTTA~\cite{cotta}, NOTE~\cite{note}, and RoTTA~\cite{rotta} extend these ideas to non-stationary streams through teacher--student updates, instance-aware normalization, and robust batch statistics. Every method in this family is built around a batch: it either updates normalization statistics estimated over the batch or minimizes a loss averaged over it. LAME~\cite{lame} is the closest to label-free output correction, refining predictions through a Laplacian objective on the batch affinity graph without touching parameters. At $B{=}1$, on norm-free or frozen-normalization tabular networks, the entropy methods update nothing and LAME's affinity graph collapses to a single isolated node, which is the degeneracy our setting exposes.

\paragraph{Calibration and post-hoc output correction} A separate line adjusts a frozen model's outputs after training. Temperature scaling and Platt scaling~\cite{guo} divide the logits by a learned positive scalar to improve calibration, leaving the decision boundary and the ranking unchanged. Calibration of dynamic neural classifiers remains an active topic in the soft-computing literature, for example, through uncertainty-aware early-exit ensembles that use a last-layer Laplace approximation and implicit ensembling to improve confidence calibration~\cite{asoc_uq_calib}. PLOC is the additive counterpart: it \emph{shifts} the logit origin rather than rescaling it, which is what moves the operating point and the predicted-positive rate (the fraction of samples labeled positive). Unlike temperature scaling, which is fit on a labeled validation set, PLOC estimates its shift from unlabeled test logits alone; and unlike a rescaling it can change which side of the threshold a borderline case falls on, recovering accuracy and F$_1$ rather than calibration alone.

\paragraph{Label shift, covariate shift, and logit adjustment} Classical label-shift estimators, including the EM procedure of Saerens et al.~\cite{saerens}, BBSE~\cite{bbse}, and MLLS/RLLS~\cite{alexandari,rlls}, estimate the target class prior from a batch of unlabeled data and reweight posteriors accordingly. Related corrections for covariate shift adapt the input-distribution mismatch directly, for instance through surrogate kernels~\cite{asoc_covshift}, and streaming classifiers handle changing class priors and imbalance with resampling-based ensembles~\cite{asoc_imb_stream}. Logit adjustment~\cite{menon} shifts logits by \emph{known} class log-frequencies to handle long-tailed training. PLOC differs in kind from both: it estimates no class prior, assumes no known frequencies, and applies a single unlabeled scalar that is not tied to any prior estimate. This is what makes it well posed at $B{=}1$, where a prior estimate from one example is a single indicator; it is also what makes PLOC an operating-point correction rather than a label-shift estimator, a distinction we make precise and probe empirically in Section~\ref{sec:scope}.
\paragraph{Online output and classifier adjustment} A few methods adapt outputs sample by sample. T3A~\cite{t3a} adjusts the classifier online through running pseudo-prototypes, and DUA~\cite{dua} updates normalization statistics online with a small momentum. In the data-stream literature, adaptive online incremental learners similarly track evolving distributions, but they do so through drift detectors and continual parameter updates~\cite{asoc_online_inc}. PLOC is more restrictive than both: it alters neither prototypes nor normalization, only the origin of the classifier's output logits, and it carries one number of state per logit. This restriction makes the method well defined when only one sample is available, and it is also what gives the deferred variant an exact ranking guarantee that prototype- and normalization-based updates cannot offer.

\paragraph{Deep learning for tabular data} Although gradient-boosted trees remain strong on tabular problems, deep architectures have closed much of the gap, and deep tabular models have been studied in the risk domains that motivate our benchmarks, such as credit scoring with deep multiple kernel learning~\cite{asoc_credit}. We use three representative families as source models: a plain multilayer perceptron, the attention-based FT-Transformer~\cite{fttransformer}, and the contextual-embedding TabTransformer~\cite{tabtransformer}. We report every result across all three so that the singleton pathology and the origin-centering fix are seen to be properties of the deployment regime rather than of one architecture.

\paragraph{Streaming and prequential evaluation} Related work in the soft-computing literature studies online classification of evolving streams, including clustering-and-ensemble methods for gradual and abrupt concept drift~\cite{asoc_stream_ens} and semi-supervised active ensembles for imbalanced streams with limited labels~\cite{asoc_active_stream}. Finally, the protocol we adopt, predicting on each example before it can be used for adaptation, is the prequential or test-then-train scheme long studied in online and stream learning~\cite{gama}. The term \emph{prequential} in our method name reflects this discipline directly: the running mean at step $t$ uses only strictly past logits, so every prediction is made before the example it is scored on contributes anything to the model. This is what keeps the singleton evaluation honest and rules out label or future-sample leakage.

\section{Problem setting and the singleton regime}
Let $f_{\theta_0}$ be a frozen {source classifier producing, for an input $x$, a logit vector $z(x)\in\mathbb{R}^K$ over the $K$ classes and a predictive distribution $\mathrm{softmax}(z(x))$}. Target examples arrive as a stream $x_1, x_2, \dots, x_T$. At step $t$, a prediction must be emitted using only $\{x_i\}_{i\le t}$ and no labels. We write $\mathcal{P}_S$ and $\mathcal{P}_T$ for the source and target joint distributions. In this setting the source and target distributions are not the same. The defining feature of the regime is the batch size $B=1$: exactly one example is observed before a decision is required. This is more restrictive than online adaptation with small batches, because the batch-level quantities that FTTA methods generally rely on are not only noisy, but often \emph{undefined or uninformative} at batch size $B=1$. The central design question is not how to shrink a batch method to one sample, but which quantity can be estimated reliably when only one sample is available at a time. We take that quantity to be the location of the score stream, formalized next.

\section{Method}
\label{sec:method}

\subsection{Source model}
We start from a classifier $f_{\theta_0}$ trained on the source distribution $\mathcal{P}_S$. For an input $x$ the network produces a logit vector $z(x)\in\mathbb{R}^K$ over the $K$ classes and a predictive distribution $p(x)=\mathrm{softmax}(z(x))$. Because adding a common constant to all $K$ logits leaves the softmax unchanged, only relative logits carry information, and for $K{=}2$ the network can equivalently emit the single relative logit.
\begin{equation}
z(x) = f_{\theta_0}(x) = \log\frac{p_{\theta_0}(y{=}1\mid x)}{p_{\theta_0}(y{=}0\mid x)},
\end{equation}
with class-one probability $p(x)=\sigma(z(x))$, where $\sigma(u)=1/(1+e^{-u})$ is the logistic function.

All TableShift tasks in our evaluation have $K{=}2$, matching the evaluation protocol of FTAT~\cite{ftat} and PFT3A~\cite{pft3a}, so we use this scalar form throughout and note the general form of each PLOC equation where it differs. The default prediction thresholds $p(x)$ at the operating point chosen on source data, that is $z(x)>0$.

\subsection{Source model training phase}
The parameters $\theta_0$ are obtained prior to deployment by empirical risk minimization of the cross-entropy $\mathbb{E}_{(x,y)\sim\mathcal{P}_S}[-\log p_\theta(y\mid x)]$ on labeled source data, and remain fixed thereafter. In the case of two classes, this scalar-logit formulation reads
\begin{equation}
\theta_0 = \arg\min_{\theta}\ \mathbb{E}_{(x,y)\sim\mathcal{P}_S}\!\big[-y\log\sigma(f_\theta(x)) - (1{-}y)\log\!\big(1-\sigma(f_\theta(x))\big)\big].
\end{equation}
We consider 3 distinct standard tabular architectures for the classifier $f_{\theta_0}$, i.e., an MLP, an FT-Transformer, and a TabTransformer (Section~\ref{sec:results}). After this phase the model is \emph{frozen}: test-time adaptation never updates $\theta_0$, computes no gradients, and stores no optimizer state. Everything that follows operates only on the scalar outputs $z(x)$.

\subsection{Proposed approach}
\paragraph{Prequential logit-origin centering (PLOC)} Under the train-test distribution shift, the source model operating point is often miscentered: the ordering induced by $z(x)$ is still informative, but the threshold sits in the wrong place. The proposed PLOC approach corrects this online by subtracting, at each step, the running mean of the logits seen so far. The prediction rule is
\begin{equation}
\hat p_t = \sigma\!\left(z_t - \mu_{t-1}\right),
\label{eq:ploc}
\end{equation}
where the centering term is the mean of \emph{strictly past} logits,
\begin{equation}
\mu_{t-1} = \frac{1}{t-1}\sum_{i<t} z_i .
\end{equation}
The current sample never participates in its own centering. The running mean admits the $O(1)$ update $\mu_t = \frac{(t-1)\mu_{t-1}+z_t}{t}$, so PLOC needs to store only a single running number, uses no labels, estimates no class prior, and updates no model parameters. For $K>2$, we call the coordinate-wise multiclass extension \emph{PLOC-K}, where $K$ is the number of classes rather than a tunable parameter. It maintains an accumulator $s\in\mathbb{R}^K$, initialized to zero. Before updating $s$, it emits $\hat p_t=\mathrm{softmax}\!\big(z_t-s/(t-1)\big)$ for $t>1$ (using zero centering at $t=1$), predicts $\arg\max_k \hat p_{t,k}$, and then sets $s\leftarrow s+z_t$. The centering term $s/(t-1)$ is exactly the running mean $\mu_{t-1}\in\mathbb{R}^K$ of past logit vectors applied coordinate-wise, so the state is $K$ scalars, and for $K{=}2$ the rule reduces to Eq.~\eqref{eq:ploc}. Figure~\ref{fig:method} depicts the application of PLOC, and Algorithm~\ref{alg:ploc} describes the complete procedure.

\begin{algorithm}
\caption{Prequential Logit-Origin Centering (PLOC). The method stores only the running mean of past logits; no target labels; no tuning.}
\label{alg:ploc}
{\small\begin{tabbing}
\quad\=\quad\=\kill
1:\> $s \gets 0$ \quad{\footnotesize// accumulator} \\
2:\> \textbf{for} $t = 1, 2, \dots$ as $x_t$ arrives \textbf{do} \\
3:\> \> $z_t \gets f_{\theta_0}(x_t)$ \quad{\footnotesize// source logit} \\
4:\> \> $\mu_{t-1} \gets s/(t{-}1)$ if $t>1$, else $0$ \\
5:\> \> \textbf{emit} $\hat p_t \gets \sigma(z_t - \mu_{t-1})$ \\
6:\> \> $s \gets s + z_t$ \\
7:\> \textbf{end for}
\end{tabbing}}
\end{algorithm}

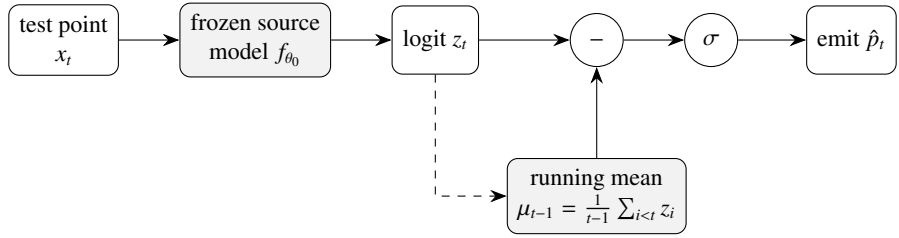
\begin{figure}[!htbp]
\centering
\begin{tikzpicture}[
  font=\small,
  box/.style={draw, rounded corners, minimum height=9mm, align=center, inner sep=4pt},
  src/.style={box, fill=black!5},
  op/.style={draw, circle, minimum size=7mm, inner sep=0pt},
  >={Stealth[length=2.2mm]}
]
\node[box] (x) {test point\\$x_t$};
\node[src, right=8mm of x] (f) {frozen source\\model $f_{\theta_0}$};
\node[box, right=8mm of f] (z) {logit $z_t$};
\node[op, right=12mm of z] (sub) {$-$};
\node[op, right=8mm of sub] (sig) {$\sigma$};
\node[box, right=9mm of sig] (p) {emit $\hat p_t$};
\node[src, below=12mm of sub, align=center] (mu) {running mean\\$\mu_{t-1}=\tfrac{1}{t-1}\sum_{i<t}z_i$};
\draw[->] (x)--(f);
\draw[->] (f)--(z);
\draw[->] (z)--(sub);
\draw[->] (sub)--(sig);
\draw[->] (sig)--(p);
\draw[->] (mu)--(sub);
\draw[->, dashed] (z) |- (mu);
\end{tikzpicture}
\caption{The proposed PLOC applied to perform inference on test point $x_t$. The frozen source model maps the incoming example $x_t$ to a logit $z_t$. PLOC subtracts the running mean $\mu_{t-1}$ of previously observed logits before the sigmoid and emits $\hat p_t=\sigma(z_t-\mu_{t-1})$. The current logit then updates the running mean (dashed) for the next step. No labels are used, and no model parameters change. The only quantity the method stores between the steps is a single number, the running mean of the logits seen so far.}
\label{fig:method}
\end{figure}

\paragraph{Deferred PLOC} When predictions may be produced after a single unlabeled pass over the stream, as in batch-review tabular workflows where a day's transactions are scored together overnight, we center by the full-stream mean instead of the prequential one. The shift is then a single constant,
\begin{equation}
\mu_T = \frac{1}{T}\sum_{i=1}^{T} z_i ,
\end{equation}
applied identically to every score,
\begin{equation}
\hat p_i^{\mathrm{def}} = \sigma\!\left(z_i - \mu_T\right).
\end{equation}
Because one constant is subtracted from every logit, deferred PLOC is the simplest variant to analyze. Section~\ref{sec:theory} shows that the proposed deferred PLOC leaves the ranking untouched exactly, in the two-class case. For $K>2$, this holds for every one-vs-one class comparison but not for one-vs-rest rankings (Remark~\ref{rem:multiclass}). Deferred PLOC is therefore an offline variant whose information access matches a batch method run at $B{=}T$. Unlike the batch methods, however, it uses that access only to estimate a single scalar shift, updates no parameters, and applies one monotone transform, which is what yields the exact ranking guarantee of Proposition~\ref{prop:auroc}. We report it as the analyzable anchor of the method family. PLOC remains the singleton method, and Section~\ref{sec:results} shows the two match closely, so restricting to strictly online operation costs almost nothing.

\paragraph{Interpretation} The proposed PLOC approach estimates only the score-origin shift rather than class priors, neighborhoods, or adaptation gradients. It introduces no hyperparameters: there is no momentum, window, temperature, learned scalar, or tuned threshold. We keep this form fixed throughout the paper, and Section~\ref{sec:results} shows it is still able to outperform the compared methods in the unit batch size FTTA problem setting. PLOC is an operating-point correction and not a label-shift estimator, and we describe in Section~\ref{sec:scope} where that distinction matters.

\paragraph{A probability-space variant (PLOC-prob)} The same centering idea can also be applied after the sigmoid, directly to the probabilities. At each step, this variant subtracts the running mean of the past probabilities and re-centers the result around $\tfrac12$: $\hat p_t = \mathrm{clip}(p_t - \bar p_{<t} + 0.5,\,0,\,1)$, where $p_t=\sigma(z_t)$ and $\bar p_{<t}$ is the running mean of past probabilities. The clip operation keeps the output inside $[0,1]$. This variant shares the centering mechanism of PLOC, but because probabilities are a bounded, nonlinear transform of the logits, it loses the monotone-in-logit structure behind the exact ranking guarantee and needs the clip as an extra safeguard. We report this variant as \emph{PLOC-prob} in the results.

\section{Theoretical analysis}
\label{sec:theory}
We analyze the scalar instantiation of Section~\ref{sec:method}, which covers every tabular benchmark in our evaluation. For the general $K$-class form, the exact guarantee below extends to every pairwise (one-vs-one) class comparison but not to one-vs-rest rankings computed from softmax probabilities. Remark~\ref{rem:multiclass} makes this precise and Section~\ref{sec:vision} measures it. Recall that for a scoring rule $s(\cdot)$ the area under the ROC curve can be written as the probability that a random positive outscores a random negative,
\begin{equation*}
\begin{aligned}
\mathrm{AUROC}(s) ={}& \Pr\big(s(X^+) > s(X^-)\big)
+ \tfrac{1}{2}\Pr\big(s(X^+){=}s(X^-)\big),
\end{aligned}
\end{equation*}
with $X^+,X^-$ drawn from the target positive and negative class-conditionals. AUROC depends on the scores only through their ordering.

\begin{proposition}[Deferred PLOC preserves AUROC exactly]
\label{prop:auroc}
For any fixed $\mu_T\in\mathbb{R}$, the deferred-PLOC scores $s_i=\sigma(z_i-\mu_T)$ induce the same ROC curve as the source logits $z_i$, and therefore $\mathrm{AUROC}(s)=\mathrm{AUROC}(z)$.
\end{proposition}
\begin{proof}
The map $z\mapsto\sigma(z-\mu_T)$ is a strictly increasing bijection of $\mathbb{R}$ onto $(0,1)$, so for every pair it preserves both $z_i>z_j$ and $z_i=z_j$. AUROC depends on the scores only through these orderings, hence $\mathrm{AUROC}(s)=\mathrm{AUROC}(z)$ and the ROC curves coincide.
\end{proof}

\begin{remark}[The multi-class case]
\label{rem:multiclass}
For $K>2$, deferred PLOC subtracts a constant vector $\mu_T\in\mathbb{R}^K$ from every logit vector. For any pair of classes $(k,j)$ the pairwise log-odds shifts by a constant, $\log \hat p^{\mathrm{def}}_k - \log \hat p^{\mathrm{def}}_j = (z_k - z_j) - (\mu_{T,k}-\mu_{T,j})$, so by the argument of Proposition~\ref{prop:auroc} every pair-conditional (one-vs-one) AUROC---computed over the examples whose true class lies in $\{k,j\}$, ranked by this pairwise log-odds (equivalently by $z_k-z_j$)---is preserved exactly. The guarantee is specific to this pairwise-logit score; a one-vs-one AUROC computed from the full-softmax probability of class $k$ is not covered, since that probability depends on all logit coordinates. One-vs-rest rankings computed from softmax probabilities are \emph{not} preserved, because the softmax probability of class $k$ depends on all coordinates of the logit vector and is not a monotone function of $z_k$ alone. The ten-class probe of Section~\ref{sec:vision} confirms both effects: one-vs-one AUROC is unchanged to numerical precision ($|\Delta|\le 1.1\times10^{-8}$ across all $39$ stream--seed runs), while one-vs-rest AUROC shifts by up to $0.07$, predominantly upward on shifted streams.
\end{remark}

PLOC itself applies a \emph{time-varying} shift $\mu_{t-1}$, so it is not a single monotone transform and can, in principle, reorder examples. The next result bounds how much.

\begin{proposition}[PLOC ranking-drift bound]
\label{prop:drift}
Fix a target stream of length $T$, and let $s_t=\sigma(z_t-\mu_{t-1})$ be the PLOC score. For a positive--negative pair $(i,j)$, define
\[
D_{ij}=z_i-z_j,\qquad M_{ij}=\mu_{i-1}-\mu_{j-1},
\]
and let $\psi(u)=\mathbf{1}\{u>0\}+\tfrac{1}{2}\mathbf{1}\{u=0\}$ be the pairwise AUROC contribution. If PLOC changes the contribution of pair $(i,j)$, that is,
\[
\psi(D_{ij}-M_{ij})\neq \psi(D_{ij}),
\]
then
\[
|D_{ij}|\le |M_{ij}|.
\]
Consequently, with
\[
\mathcal{I}_T=\{(i,j): y_i{=}1,\,y_j{=}0,\ |z_i-z_j|\le|\mu_{i-1}-\mu_{j-1}|\},
\]
and with $n_+>0$ positives and $n_->0$ negatives,
\[
\left|\widehat{\mathrm{AUROC}}_T(s)-\widehat{\mathrm{AUROC}}_T(z)\right|
\le
\frac{|\mathcal{I}_T|}{n_+n_-}.
\]
Moreover, along a growing stream, suppose $n_+/T\to\pi_+\in(0,1)$, $n_-/T\to\pi_-\in(0,1)$, $\mu_t\to m$, and the empirical source-score margin condition
\[
\lim_{\varepsilon\downarrow 0}\limsup_{T\to\infty}
\frac{\#\{(i,j): y_i{=}1,\,y_j{=}0,\ |z_i-z_j|\le\varepsilon\}}{n_+n_-}
=0
\]
holds. Then $|\mathcal{I}_T|/(n_+n_-)\to 0$, and therefore
\[
\left|\widehat{\mathrm{AUROC}}_T(s)-\widehat{\mathrm{AUROC}}_T(z)\right|\to 0.
\]
\end{proposition}
\begin{proof}
Because $\sigma$ is strictly increasing, the pairwise ordering of $s_i$ and $s_j$ is the ordering of $z_i-\mu_{i-1}$ and $z_j-\mu_{j-1}$. Thus PLOC contributes $\psi(D_{ij}-M_{ij})$ for pair $(i,j)$, whereas the source contributes $\psi(D_{ij})$. If $D_{ij}>0$ and the contribution changes, then $D_{ij}-M_{ij}\le 0$, so $M_{ij}\ge D_{ij}>0$ and $|D_{ij}|\le |M_{ij}|$. If $D_{ij}<0$ and the contribution changes, then $D_{ij}-M_{ij}\ge 0$, so $M_{ij}\le D_{ij}<0$ and again $|D_{ij}|\le |M_{ij}|$. If $D_{ij}=0$, the inequality is immediate. Therefore every pair whose contribution changes belongs to $\mathcal{I}_T$. Since each pairwise contribution changes by at most one and empirical AUROC averages over $n_+n_-$ positive--negative pairs, the finite-sample bound follows.

For the asymptotic statement, fix $\varepsilon>0$. Since $\mu_t\to m$, there exists $N$ such that $|\mu_t-m|\le\varepsilon/2$ for all $t\ge N$. Hence, for all pairs with $i,j>N$, $|\mu_{i-1}-\mu_{j-1}|\le\varepsilon$. Any such pair in $\mathcal{I}_T$ must therefore satisfy $|z_i-z_j|\le\varepsilon$. Pairs with $i\le N$ or $j\le N$ are $O(T)$ in number, while $n_+n_-=\Theta(T^2)$ under the assumed class proportions, so their normalized contribution vanishes. The remaining contribution is bounded by the empirical fraction of positive--negative pairs with $|z_i-z_j|\le\varepsilon$, which vanishes as $\varepsilon\downarrow0$ by the stated margin condition. Thus $|\mathcal{I}_T|/(n_+n_-)\to0$, proving the claim.
\end{proof}

\begin{proposition}[The prequential mean is the minimum-variance origin estimate]
\label{prop:variance}
Among all unbiased causal linear estimators $c_t=\sum_{i<t}w_i z_i$ of a stationary logit origin, with weights summing to one ($\sum_{i<t}w_i=1$) and independent logit noise of variance $\sigma^2$, the uniform prequential mean ($w_i=1/(t-1)$) attains the smallest variance.
\end{proposition}
\begin{proof}
The variance is $\sigma^2\sum_{i<t}w_i^2$, and Cauchy--Schwarz gives $\sum_{i<t}w_i^2\ge 1/(t-1)$ on the simplex $\sum_{i<t}w_i=1$, with equality only for the uniform weights.
\end{proof}

Proposition~\ref{prop:variance} is the theoretical counterpart of the design ablation in Section~\ref{sec:ablation}: there, an exponential moving average, a sliding window, a running $z$-score, and a Huber-clipped mean are all empirically no better than the plain prequential mean, and the faster-forgetting variants are slightly worse, as minimum variance predicts for a stationary stream.

\paragraph{Self-modulation.} The applied shift $\mu_T$ is observable from the unlabeled stream and requires no tuning: it is the source's mean operating point on the target. Across our datasets it spans a wide range, from $\mu_T\approx-1.1$ on HELOC and $-0.6$ on ACS Public Coverage, where the source under-predicts the positive class, to $\mu_T\approx+1.1$ on ANES and ASSISTments, where the target leans positive. PLOC removes this origin shift, and its \emph{benefit} concentrates where the source operating point is far from optimal for threshold metrics (HELOC, ACS Public Coverage; more mildly ASSISTments). Where the source is already well positioned (ANES), PLOC is near-neutral even though the $|\mu_T|$ is not small: the gain is governed by how miscentered the operating point is, not by $|\mu_T|$ alone. The correction is computed from the unlabeled target stream, providing a tuning-free, ranking-preserving adjustment that helps when the operating point is wrong and has little effect otherwise.

\section{Experimental setup}
\paragraph{Datasets} We use five tabular benchmarks from TableShift~\cite{tableshift} under their out-of-distribution splits, chosen to cover credit, healthcare, education, and survey domains with genuine source-to-target drift. All five datasets are two-class prediction tasks, and in each, the target stream comes from a domain held out from source training by the TableShift split.
\begin{itemize}
\item \emph{HELOC}~\cite{heloc}: It comprises of anonymized credit-bureau records of home-equity line-of-credit applicants (financial activity, balances, credit inquiries, delinquency history) from the FICO Explainable Machine Learning Challenge. The task is to predict repayment performance over the two years after account opening, where the positive class repays as negotiated and the negative class becomes at least $90$ days delinquent. The domain split is on a consolidated external risk score, with the target stream containing the applicants whose score falls on the side of the threshold unseen during training.
\item \emph{Diabetes Readmission}~\cite{strack}: It comprises of clinical records of diabetic-patient encounters at $130$ US hospitals (demographics, diagnoses, laboratory tests, medications). The task is to predict whether a patient is readmitted to hospital after discharge. The target domain consists of emergency-room admissions.
\item \emph{ASSISTments}~\cite{assistments}: It comprises of interaction logs from the ASSISTments online mathematics tutoring platform (student performance history, problem attributes, affect indicators). The task is to predict whether a student answers the next problem correctly. The target domain consists of schools not present in the source data.
\item \emph{ANES}~\cite{anes}: It comprises of survey responses from the American National Election Studies Time Series Cumulative Data File (political attitudes, media consumption, demographics). The task is to predict whether a respondent votes in the US national election. The target domain is the Southern census region, excluded from training.
\item \emph{ACS Public Coverage}~\cite{folktables}: It comprises of person-level demographic and socioeconomic records from the American Community Survey, restricted to low-income individuals under $65$. The task is to predict whether an individual holds public health-insurance coverage. The target domain consists of respondents with disabilities.
\end{itemize}
The target stream sizes are $6914$, $50968$, $1906$, $21231$, and $817877$, respectively, for the HELOC, Diabetes Readmission, ASSISTments, ANES and ACS Public Coverage datasets.

\paragraph{Backbones and training} We evaluate three architectures spanning the common tabular deep-learning families: a multilayer perceptron (MLP), an FT-Transformer, and a TabTransformer. Each is trained on source data with five independent initialization seeds ($20$ epochs, batch size $1024$, learning rate $0.01$, weight decay $0.01$), yielding five independently trained source models per architecture. We refer to each of these five frozen trained models as a \emph{source checkpoint}. All adaptation is then performed on these frozen checkpoints.

\paragraph{Protocol} All test-time adaptation is run at \textbf{batch size one}. No target labels are used by any method at any point during adaptation, which we assert programmatically for every run. The classifier produces a probability $\hat p_t\in(0,1)$ for every stream element. The corresponding hard prediction is $\hat y_t=\mathbf{1}\{\hat p_t\ge \tfrac12\}$. Ground-truth target labels $y_t$ are used only in this evaluation stage for computing the performance metrics, and never during adaptation. Let $T$ refer to the stream length and $\mathrm{TP},\mathrm{TN},\mathrm{FP},\mathrm{FN}$ refer to the number of true positives, true negatives, false positives, and false negatives induced by $\hat y$. We report accuracy, balanced accuracy, F$_1$ score, AUROC, $15$-bin expected calibration error (ECE), negative log-likelihood (NLL), and the Brier score, averaged over the five datasets and five source checkpoints ($25$ runs per reported value). We compare against source (no test-time adaptation), Tent~\cite{tent}, EATA~\cite{eata}, SAR~\cite{sar}, LAME~\cite{lame}, FTAT~\cite{ftat}, and PFT3A~\cite{pft3a}. The target performance metrics are computed as follows.
\begin{itemize}
\item \emph{Accuracy (Acc)} is the fraction of correct hard predictions, $\frac{1}{T}\sum_{t=1}^{T}\mathbf{1}\{\hat y_t=y_t\}$.
\item \emph{Balanced accuracy (BAcc)} is the mean of the two per-class recalls, $\tfrac12\big(\frac{\mathrm{TP}}{\mathrm{TP}+\mathrm{FN}}+\frac{\mathrm{TN}}{\mathrm{TN}+\mathrm{FP}}\big)$, which is insensitive to class imbalance.
\item \emph{F$_1$ score} is the harmonic mean of precision and recall for the positive class, $\mathrm{F}_1=\frac{2\,\mathrm{TP}}{2\,\mathrm{TP}+\mathrm{FP}+\mathrm{FN}}$, set to $0$ when the denominator vanishes.
\item \emph{AUROC} is the area under the ROC curve traced by sweeping a decision threshold over the predicted probabilities $\hat p$. As recalled in Section~\ref{sec:theory}, this is equivalent to the Wilcoxon--Mann--Whitney statistic $\Pr(\hat p^+>\hat p^-)+\tfrac12\Pr(\hat p^+=\hat p^-)$ where $\hat p^+$ and $\hat p^-$ denote the scores of randomly drawn positive and negative pairs, respectively. Unlike threshold-dependent metrics, AUROC relies strictly on the relative ranking of the predicted probabilities rather than their absolute values.
\item \emph{ECE} measures calibration of the prediction confidence $c_t=\max(\hat p_t,1-\hat p_t)$. The confidence range is partitioned into $M{=}15$ equal-width bins $B_1,\dots,B_M$, and $\mathrm{ECE}=\sum_{m=1}^{M}\frac{|B_m|}{T}\big|\mathrm{acc}(B_m)-\mathrm{conf}(B_m)\big|$, where $\mathrm{acc}(B_m)$ is the accuracy of the hard predictions whose confidence falls in bin $B_m$ and $\mathrm{conf}(B_m)$ is their mean confidence~\cite{guo}.
\item \emph{NLL} is the mean negative log-likelihood of the true labels, $-\frac{1}{T}\sum_{t=1}^{T}\big[y_t\log\hat p_t+(1-y_t)\log(1-\hat p_t)\big]$, with probabilities clipped to $[10^{-6},\,1-10^{-6}]$ for numerical stability.
\item \emph{Brier score} is the mean squared error between the predicted probability and the label, $\frac{1}{T}\sum_{t=1}^{T}(\hat p_t-y_t)^2$.
\end{itemize}

\begin{table*}[t]
\centering
\caption{Test-time adaptation at batch size one, for different backbones, averaged over 5 datasets and 5 source checkpoints. Best results under each performance metric are shown in \textbf{bold}. AUROC closest to source is preferred. ${}^\dagger$The official PFT3A implementation is undefined at $B{=}1$, so we evaluate our own guarded singleton adaptation of it. At larger batch sizes our integration reproduces its published results (Section~\ref{sec:batch}).}
\label{tab:main}
\small
\begin{tabular}{llccccccc}
\toprule
Backbone & Method & Acc$\uparrow$ & BAcc$\uparrow$ & F$_1$$\uparrow$ & AUROC$\uparrow$ & ECE$\downarrow$ & NLL$\downarrow$ & Brier$\downarrow$ \\
\midrule
\multirow{9}{*}{MLP}
 & Source & 0.612 & 0.638 & 0.580 & 0.733 & 0.160 & 0.732 & 0.252 \\
 & Tent & 0.612 & 0.638 & 0.580 & 0.733 & 0.160 & 0.732 & 0.252 \\
 & EATA & 0.612 & 0.638 & 0.580 & 0.733 & 0.160 & 0.732 & 0.252 \\
 & SAR & 0.612 & 0.638 & 0.580 & 0.733 & 0.160 & 0.732 & 0.252 \\
 & LAME & 0.612 & 0.638 & 0.580 & 0.733 & 0.160 & 0.732 & 0.252 \\
 & FTAT & 0.589 & 0.604 & 0.580 & 0.660 & 0.066 & 0.656 & 0.232 \\
 & PFT3A$^\dagger$ & 0.576 & 0.606 & 0.463 & 0.707 & 0.214 & 1.014 & 0.296 \\
& \textbf{PLOC} & \textbf{0.663} & \textbf{0.670} & \textbf{0.694} & \textbf{0.732} & 0.060 & \textbf{0.611} & 0.210 \\
 & PLOC-prob & 0.659 & 0.669 & 0.683 & 0.732 & \textbf{0.056} & 0.646 & \textbf{0.209} \\
\midrule
\multirow{9}{*}{FT-Transformer}
 & Source & 0.613 & 0.637 & 0.539 & 0.718 & 0.160 & 0.734 & 0.258 \\
 & Tent & 0.613 & 0.637 & 0.539 & 0.718 & 0.160 & 0.734 & 0.258 \\
 & EATA & 0.613 & 0.637 & 0.539 & 0.718 & 0.160 & 0.734 & 0.258 \\
 & SAR & 0.610 & 0.637 & 0.535 & 0.717 & 0.162 & 0.736 & 0.259 \\
 & LAME & 0.613 & 0.637 & 0.539 & 0.718 & 0.160 & 0.734 & 0.258 \\
 & FTAT & 0.593 & 0.595 & 0.561 & 0.639 & 0.098 & 0.678 & 0.243 \\
 & PFT3A$^\dagger$ & 0.583 & 0.552 & 0.605 & 0.569 & 0.162 & 0.928 & 0.279 \\
  & \textbf{PLOC} & \textbf{0.652} & \textbf{0.662} & \textbf{0.682} & 0.718 & \textbf{0.056} & \textbf{0.605} & 0.212 \\
 & PLOC-prob & 0.651 & 0.661 & 0.680 & \textbf{0.718} & 0.058 & 0.610 & \textbf{0.209} \\
\midrule
\multirow{9}{*}{TabTransformer}
 & Source & 0.573 & 0.605 & 0.500 & 0.691 & 0.157 & 0.731 & 0.257 \\
 & Tent & 0.573 & 0.605 & 0.500 & 0.691 & 0.157 & 0.731 & 0.257 \\
 & EATA & 0.573 & 0.605 & 0.500 & 0.691 & 0.157 & 0.731 & 0.257 \\
  & SAR & 0.573 & 0.605 & 0.500 & 0.691 & 0.157 & 0.731 & 0.257 \\
 & LAME & 0.573 & 0.605 & 0.500 & 0.691 & 0.157 & 0.731 & 0.257 \\
 & FTAT & 0.526 & 0.546 & 0.454 & 0.599 & 0.113 & 0.721 & 0.260 \\
 & PFT3A$^\dagger$ & 0.533 & 0.566 & 0.425 & 0.616 & 0.250 & 2.278 & 0.328 \\
  & \textbf{PLOC} & 0.629 & 0.642 & \textbf{0.696} & 0.690 & 0.039 & \textbf{0.614} & \textbf{0.213} \\
 & PLOC-prob & \textbf{0.636} & \textbf{0.644} & 0.659 & \textbf{0.691} & \textbf{0.037} & 0.624 & 0.213 \\
\bottomrule
\end{tabular}
\end{table*}

The evaluated metrics capture distinct aspects of model performance. Accuracy, balanced accuracy, and the $\text{F}_1$-score are threshold-dependent metrics, making them highly sensitive to the choice of operating point. ECE is a binned calibration summary, while NLL and the Brier score are proper scoring rules that reflect both calibration and discrimination. Finally, AUROC isolates discrimination performance by focusing strictly on score ranking. For the calibration and error metrics (ECE, NLL, and Brier score), lower values indicate superior performance, whereas higher values are preferred for all other metrics.

\section{Results}
\label{sec:results}
\paragraph{Main comparison} Table~\ref{tab:main} reports the test-time adaptation performance at batch size one, for different backbones, averaged over five datasets and five source checkpoints. PLOC improves the threshold and probabilistic metrics over both the source model and FTAT while leaving AUROC essentially unchanged, while the entropy- and normalization-based methods (Tent, EATA, SAR, and LAME) return the source predictions unchanged or nearly unchanged. With the MLP backbone, PLOC outperforms FTAT by an absolute margin of $7.4\%$ accuracy and by $11.4$ F$_1$ points. It matches the source AUROC very closely ($0.732$ vs.\ $0.733$), achieves a 62.5\% reduction in ECE (from $0.160$ to $0.060$), and lowers NLL from $0.732$ to $0.611$. We observe a similar pattern on the FT-Transformer backbone. 

The probability-space variant (PLOC-prob) achieves nearly the same calibration as PLOC. Since the two methods share the centering step and differ only in whether it is applied to the logits or to the probabilities, this similarity shows that re-centering the output origin itself, rather than the particular space in which the centering is performed, is what drives the gains. With the TabTransformer backbone, PLOC-prob is marginally ahead of PLOC on accuracy, balanced accuracy, and ECE, though behind on F$_1$ and NLL. The two variants share the centering mechanism and differ only in parameterization, so their gains land close together, and the better variant depends on the target metric. We recommend the logit form (PLOC) whenever the ranking must be protected, since probability-space centering (PLOC-prob) carries no ranking guarantee. The ranking cost of PLOC-prob is visible in the generalization probe of Section~\ref{sec:vision}, where it lowers AUROC ($0.902\to0.895$) while PLOC leaves it unchanged at $0.902$; the exact preservation guarantee of Proposition~\ref{prop:auroc} applies to deferred PLOC. The same behavior recurs on the third backbone, i.e., the TabTransformer. Even in this case, Tent, EATA, SAR, and LAME return the source predictions unchanged (SAR to within numerical precision), since the quantities that drive their updates have nothing to act on at batch size $B{=}1$. FTAT and PFT3A again trade away ranking (AUROC $0.599$ and $0.616$ against the source's $0.691$), while PLOC leaves AUROC essentially unchanged ($0.690$ vs.\ $0.691$) and lifts F$_1$ from $0.50$ to $0.696$ and reduces ECE from $0.157$ to $0.039$. The same degradation of the compared test-time adaptation methods and the same correction by PLOC thus appear on all three architectures, including the weakest-ranking one (TabTransformer). This indicates that the miscentered operating point is a recurring consequence of the distribution shift itself, which PLOC corrects, rather than a peculiarity of any single architecture that PLOC exploits. Deferred PLOC (Table~\ref{tab:deferred}) matches source AUROC to three decimals on both backbones while retaining the threshold gains.

\begin{table*}[!t]
\centering
\caption{Offline reference: Deferred PLOC makes one unlabeled pass over the full stream before scoring (Section~\ref{sec:method}), so its information access corresponds to $B{=}T$ rather than $B{=}1$. It is therefore reported separately from the singleton comparison of Table~\ref{tab:main}, with the source rows repeated for reference. Deferred PLOC matches PLOC within $0.002$ on every metric and preserves source AUROC exactly (Proposition~\ref{prop:auroc}).}
\label{tab:deferred}
\small
\begin{tabular}{llccccccc}
\toprule
Backbone & Method & Acc$\uparrow$ & BAcc$\uparrow$ & F$_1$$\uparrow$ & AUROC$\uparrow$ & ECE$\downarrow$ & NLL$\downarrow$ & Brier$\downarrow$ \\
\midrule
\multirow{2}{*}{MLP}
 & Source & 0.612 & 0.638 & 0.580 & 0.733 & 0.160 & 0.732 & 0.252 \\
 & Deferred PLOC & 0.662 & 0.670 & 0.695 & 0.733 & 0.061 & 0.612 & 0.211 \\
\midrule
\multirow{2}{*}{FT-Transformer}
 & Source & 0.613 & 0.637 & 0.539 & 0.718 & 0.160 & 0.734 & 0.258 \\
 & Deferred PLOC & 0.653 & 0.662 & 0.684 & 0.718 & 0.058 & 0.607 & 0.213 \\
\midrule
\multirow{2}{*}{TabTransformer}
 & Source & 0.573 & 0.605 & 0.500 & 0.691 & 0.157 & 0.731 & 0.257 \\
 & Deferred PLOC & 0.628 & 0.642 & 0.697 & 0.691 & 0.040 & 0.614 & 0.213 \\
\bottomrule
\end{tabular}
\end{table*}

\paragraph{The AUROC--accuracy tradeoff} Figure~\ref{fig:trade} plots, for each method, the mean change in accuracy against the mean change in AUROC relative to the source. PLOC and deferred PLOC appear near zero AUROC change with positive accuracy change. In contrast, batch methods (FTAT, PFT3A) sit in the lower-left, i.e., at $B{=}1$. They reduce ranking quality without consistent threshold-metric gains, with FTAT collapsing on the large ACS Public Coverage stream. In this panel, batch and weight adaptation reduce the ranking quality used by downstream thresholds and risk scores, whereas a logit-origin shift improves accuracy while leaving the ranking intact.

\begin{figure*}[!htbp]
\centering
\includegraphics[width=\textwidth]{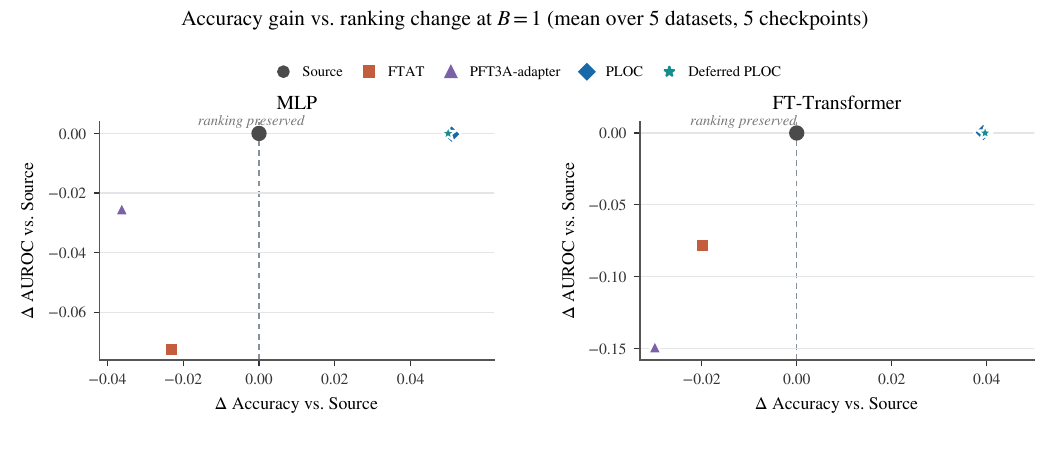}
\caption{Change in accuracy versus change in AUROC relative to the frozen source model at $B{=}1$ (mean over five datasets and five checkpoints). PLOC gains accuracy while keeping AUROC at the source value, and deferred PLOC coincides with PLOC, reflecting the ranking-preservation guarantee. The batch methods (FTAT, PFT3A-adapter) lose on both axes.}
\label{fig:trade}
\end{figure*}

\paragraph{Where the gains come from} Figure~\ref{fig:perds} breaks F$_1$ down by dataset. Performance gains are concentrated primarily in scenarios where the source model's operating point is severely misaligned. On the HELOC dataset, the source underpredicts the positive class so severely that, at the default cutoff, it labels only $16\%$ of applicants as likely to repay (the positive class), achieving $0.512$ accuracy, $0.277$ F$_1$, and $0.200$ ECE, while AUROC remains $0.684$. PLOC, observing a running logit mean of about $-0.81$, recenters the origin and raises the fraction of applicants labeled as likely to repay from $16\%$ to $55\%$, lifting the accuracy to $0.641$ and F$_1$ to $0.677$ and reducing ECE to $0.047$, with the AUROC unchanged at $0.684$ (Table~\ref{tab:perds}). With the FT-Transformer backbone, the source predicts a single class at the default threshold on HELOC ($\mathrm{F}_1=0$), which PLOC restores. On the ANES dataset, where the source operating point is already well-centered, PLOC has a negligible impact on the accuracy, consistent with the self-modulation of Section~\ref{sec:theory}.

\begin{figure*}[t]
\centering
\includegraphics[width=\textwidth]{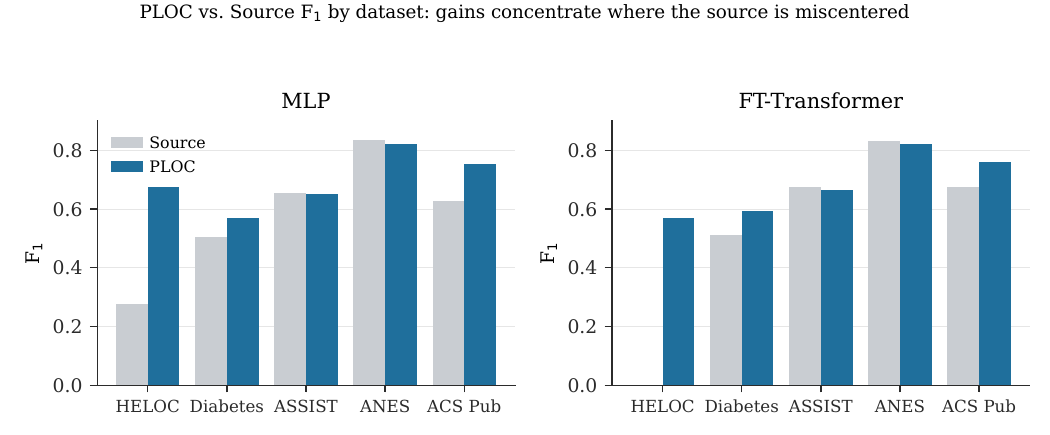}
\caption{PLOC versus source F$_1$ by dataset. Gains concentrate where the source is miscentered (HELOC, ACS Public Coverage) and PLOC is near-neutral where it is already well centered (ANES).}
\label{fig:perds}
\end{figure*}

\paragraph{Per-dataset results}
Tables~\ref{tab:perds}--\ref{tab:perds_tabt} report the source-versus-PLOC comparison per dataset for each backbone, making explicit where the aggregate gains in Table~\ref{tab:main} originate. The empirical AUROC changes are small: within $0.001$ everywhere except ASSISTments with the TabTransformer, where the change reaches $0.005$. Accuracy, F$_1$, and ECE improve most on the miscentered streams (HELOC, ACS Public Coverage) and are near-neutral where the source is already well centered (ANES).

\begin{table*}[!htbp]
\centering
\caption{Per-dataset MLP results, source $\to$ PLOC, mean over five checkpoints. AUROC is preserved to within $0.001$ on every dataset. Threshold and calibration gains concentrate on the miscentered streams.}
\label{tab:perds}
\small
\begin{tabular}{lcccc}
\toprule
Dataset & Acc$\uparrow$ & F$_1$ $\uparrow$ & AUROC$\uparrow$ & ECE$\downarrow$ \\
\midrule
HELOC & $0.512\!\to\!0.641$ & $0.277\!\to\!0.677$ & $0.684\!\to\!0.684$ & $0.200\!\to\!0.047$ \\
Diabetes Readmission & $0.588\!\to\!0.592$ & $0.507\!\to\!0.569$ & $0.627\!\to\!0.627$ & $0.123\!\to\!0.104$ \\
ASSISTments & $0.538\!\to\!0.581$ & $0.654\!\to\!0.650$ & $0.684\!\to\!0.683$ & $0.316\!\to\!0.092$ \\
ANES & $0.786\!\to\!0.785$ & $0.835\!\to\!0.823$ & $0.859\!\to\!0.859$ & $0.061\!\to\!0.040$ \\
ACS Public Coverage & $0.637\!\to\!0.717$ & $0.629\!\to\!0.753$ & $0.809\!\to\!0.809$ & $0.103\!\to\!0.015$ \\
\bottomrule
\end{tabular}
\end{table*}

\begin{table*}[!htbp]
\centering
\caption{Per-dataset FT-Transformer results, source $\to$ PLOC, mean over five checkpoints. AUROC is preserved to within $0.001$ on every dataset (HELOC's $\mathrm{F}_1=0$ source is a single-class collapse that PLOC repairs).}
\label{tab:perds_ft}
\small
\begin{tabular}{lcccc}
\toprule
Dataset & Acc & F$_1$ & AUROC & ECE \\
\midrule
HELOC & $0.431\!\to\!0.546$ & $0.000\!\to\!0.569$ & $0.561\!\to\!0.561$ & $0.327\!\to\!0.041$ \\
Diabetes Readmission & $0.607\!\to\!0.615$ & $0.513\!\to\!0.594$ & $0.658\!\to\!0.658$ & $0.039\!\to\!0.015$ \\
ASSISTments & $0.582\!\to\!0.587$ & $0.675\!\to\!0.666$ & $0.690\!\to\!0.691$ & $0.316\!\to\!0.187$ \\
ANES & $0.777\!\to\!0.790$ & $0.833\!\to\!0.822$ & $0.866\!\to\!0.866$ & $0.043\!\to\!0.024$ \\
ACS Public Coverage & $0.667\!\to\!0.722$ & $0.675\!\to\!0.760$ & $0.813\!\to\!0.813$ & $0.073\!\to\!0.014$ \\
\bottomrule
\end{tabular}
\end{table*}

\begin{table*}[!htbp]
\centering
\caption{Per-dataset TabTransformer results, source $\to$ PLOC, mean over five checkpoints. AUROC remains close to the source across all datasets. PLOC is near-neutral (slightly negative) on the already-centered ANES, the self-modulation of Section~\ref{sec:theory}.}
\label{tab:perds_tabt}
\small
\begin{tabular}{lcccc}
\toprule
Dataset & Acc & F$_1$ & AUROC & ECE \\
\midrule
HELOC & $0.488\!\to\!0.648$ & $0.216\!\to\!0.684$ & $0.693\!\to\!0.693$ & $0.229\!\to\!0.028$ \\
Diabetes Readmission & $0.566\!\to\!0.566$ & $0.306\!\to\!0.632$ & $0.595\!\to\!0.595$ & $0.063\!\to\!0.027$ \\
ASSISTments & $0.438\!\to\!0.443$ & $0.608\!\to\!0.604$ & $0.516\!\to\!0.511$ & $0.260\!\to\!0.062$ \\
ANES & $0.793\!\to\!0.780$ & $0.840\!\to\!0.807$ & $0.864\!\to\!0.864$ & $0.052\!\to\!0.039$ \\
ACS Public Coverage & $0.579\!\to\!0.709$ & $0.530\!\to\!0.751$ & $0.786\!\to\!0.786$ & $0.182\!\to\!0.039$ \\
\bottomrule
\end{tabular}
\end{table*}

\paragraph{Checkpoint stability} Beyond improving the mean performance, PLOC also reduces the variance across the five independently trained source checkpoints, in the dataset--backbone combinations where the source operating point is unstable. With the MLP backbone on the HELOC dataset, the F$_1$ standard deviation falls from $0.249$ (source) and $0.223$ (FTAT) to $0.021$ (PLOC), a tenfold reduction. With the FT-Transformer backbone on the ASSISTments dataset, FTAT's F$_1$ standard deviation is $0.189$ against PLOC's $0.009$. Aggregated over datasets, the across-seed standard deviation of F$_1$ drops from $0.057$ to $0.012$ on the MLP backbone and from $0.071$ to $0.006$ on the TabTransformer backbone. On already-stable datasets, the methods are comparable. PLOC removes operating-point instability where it exists and has a negligible effect where it does not.

\paragraph{Statistical significance} Because every result is computed on the same five independently trained source checkpoints, the runs are paired, and we can complement the reported means with paired statistical tests. Across the $25$ dataset--seed combinations per backbone, PLOC's improvements over both the source and FTAT on accuracy, F$_1$, and calibration are significant under a Wilcoxon signed-rank test: relative to the source, $\Delta$Acc and $\Delta$F$_1$ are positive with $p<10^{-2}$ and $p<3\times10^{-3}$ on all three backbones, and the ECE reduction is highly significant ($p<10^{-6}$, with ECE improving in $23$ to $25$ of the $25$ combinations). The AUROC difference between PLOC and the source is statistically indistinguishable from zero ($\Delta\mathrm{AUROC}=0.000\pm0.001$, $p\in[0.15,0.88]$), consistent with the vanishing ranking drift predicted by Proposition~\ref{prop:drift}. The exact preservation guarantee of Proposition~\ref{prop:auroc} applies to deferred PLOC and holds algebraically. Against FTAT, by contrast, PLOC's AUROC is significantly \emph{higher} ($\Delta\mathrm{AUROC}\approx+0.07$ to $+0.09$, $p<10^{-2}$ on every backbone), because FTAT trades ranking away at $B{=}1$ while PLOC does not.

\FloatBarrier
\section{Batch size confirms a distinct regime}
\label{sec:batch}
Next, we study whether the model behavior we observe for a test batch size of 1 is simply the natural extreme case of what happens with very small batches. To analyze this, we sweep the batch size from one to $1024$ on the four non-ACS benchmarks (three seeds, both backbones; see Figure~\ref{fig:batch}). We observe that PLOC and deferred PLOC are very stable: AUROC holds at the source value across the whole range, and deferred PLOC is batch-invariant by construction. The batch methods recover only once the test batches become larger. FTAT's ranking damage shrinks monotonically and nearly closes by $B{=}1024$ ($0.677\!\to\!0.701$ against source $0.707$). PFT3A cannot be run at $B{=}1$ in its official form: its class-prior and covariance estimators are defined over a batch of test points, and a covariance estimated from a single sample does not exist. Our guarded singleton adaptation, therefore, keeps its entropy-based update, skips the covariance and subspace terms whenever the batch is too small to estimate them, and guards the prior update against degenerate single-sample estimates. Even with larger batches, it is least stable at batch sizes $B\in\{4,8\}$, where its prior and covariance estimators run on far too few samples. It recovers to the source-level AUROC when the batch size $B\approx64$. At $B\ge512$, our integration reproduces PFT3A's reported batched threshold metrics on the shared datasets (MLP Accuracy/F$_1$ of $69.4/74.5$ against the published $69.3/74.2$), which confirms that the integration is faithful and that the $B{=}1$ failure is a property of the singleton regime rather than an artifact of our code. LAME inverts the pattern: at $B{=}1$ it returns the source predictions exactly unchanged, since its affinity graph consists of a single node with no neighbors to smooth over, and its Laplacian refinement progressively destroys ranking as the batch grows ($0.70\!\to\!0.57$). Singleton tabular fully test time adaptation is therefore not simply a small-batch variant of the FTTA. This setting is much harder for FTTA methods, since their batch-level estimates require batches that simply do not exist in the singleton setting, and the proposed PLOC is the only compared method that is both defined at $B{=}1$ and stable across the whole range.

\begin{figure}[!t]
\centering
\includegraphics[width=0.7\textwidth]{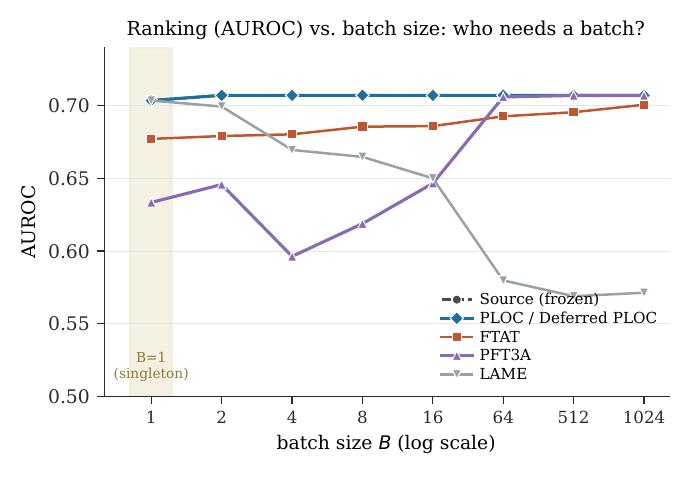}
\caption{Ranking quality versus batch size (mean AUROC over the four non-ACS datasets, both backbones, three seeds). PLOC/deferred PLOC (overlapping source) hold source AUROC at every batch size $B$. FTAT recovers gradually. PFT3A is undefined at $B{=}1$, least stable at $B\in\{4,8\}$, and recovers by $B\approx64$. LAME leaves the source predictions unchanged at $B{=}1$ but degrades as its graph grows.}
\label{fig:batch}
\end{figure}

\FloatBarrier
\section{Behavior under controlled prior shift}
\label{sec:scope}
The proposed PLOC corrects an operating point and is not a label-prior estimator. We test this distinction under deliberate, estimable prior shift. We resample each target stream to enforce positive rates $\pi_+\in\{0.1,0.3,0.5,0.7,0.9\}$ (bootstrapping the small ASSISTments stream) and evaluate at $B{=}1$ over five checkpoints and three sample seeds (Figure~\ref{fig:prior}). Across the range, PLOC and deferred PLOC track or beat the source model on balanced accuracy and calibration, and preserve AUROC exactly (deferred). Under extreme imbalance ($\pi_+=0.1$), however, they lose plain accuracy: by moving the operating point toward a balanced regime they give up the advantage that a trivial majority predictor enjoys on raw accuracy. This follows from the fact that PLOC is not a prior estimator. PLOC is the right tool when the source ranking is informative, but the operating point has drifted in the settings studied here. It is not a replacement for a label-shift estimator when the prior shift is both extreme and estimable from a batch.

\begin{figure}[!htbp]
\centering
\includegraphics[width=0.62\textwidth]{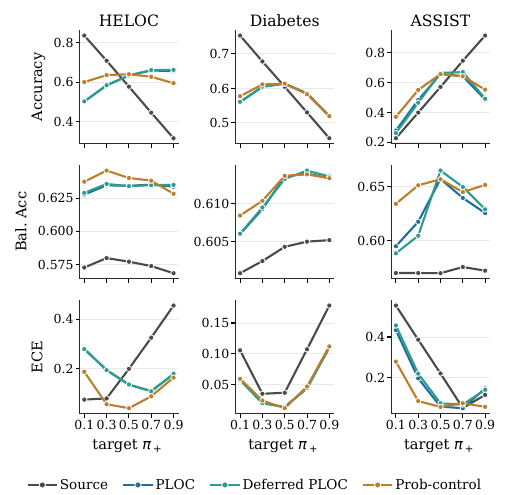}
\caption{Behavior under controlled target prior shift. PLOC and deferred PLOC track or beat the source model on balanced accuracy and ECE across the range, but under extreme imbalance ($\pi_+=0.1$), they trade plain accuracy for a balanced operating point. AUROC is preserved at every prior (deferred; omitted, $\Delta=0$).}
\label{fig:prior}
\end{figure}

\FloatBarrier
\section{Design ablations}
\label{sec:ablation}
We use PLOC in its simplest form, the prequential mean of past logits, and test whether a more elaborate origin estimator would do better. Table~\ref{tab:ablation} reports a single-seed ablation (seed $0$, MLP and FT-Transformer backbones, averaged over the five datasets) that varies only the origin estimator. The absolute numbers are not comparable to Table~\ref{tab:main}, which averages over all five checkpoints and three backbones. The comparison of interest is each variant against the plain prequential mean.

The ablation gives three observations. First, the plain prequential mean turns out to be hard to beat: no alternative estimator improves on it consistently across the reported metrics. An exponential moving average (EMA) matches it at slow decay and slightly damages ranking and calibration as the decay rate grows ($\alpha=0.2$). A sliding window over the last $100$ logits is indistinguishable from the full prequential mean. A running $z$-score (dividing by the running standard deviation as well as centering) leaves accuracy untouched while worsening calibration. Second, the only variant that improves on any metric is the Huber-clipped running mean, which lowers ECE modestly ($0.047\to0.033$) at the cost of F$_1$ and an extra clipping threshold. We prefer the hyperparameter-free mean. Third, the last row of Table~\ref{tab:ablation} reports a label-using oracle that selects the single shift minimizing the negative log-likelihood on the target stream. Even with access to the labels, it improves on PLOC's fully unlabeled estimate only marginally (F$_1$ $0.701$ versus $0.692$, ECE $0.027$ versus $0.047$), indicating that the unlabeled prequential mean is already close to the best label-selected single shift. These variants may be useful under non-stationary drift, where a forgetting estimate is the natural generalization. Under the stationary streams studied here, the minimal mean is the preferred default.

\begin{table}[!t]
\centering
\caption{Single-seed origin-estimator ablation (seed $0$, MLP and FT-Transformer backbones, averaged over the five datasets). Absolute values are not comparable to Table~\ref{tab:main}, which averages over all five checkpoints and three backbones. Each variant should be read against \emph{PLOC (prequential mean)}. The last row is an oracle that uses the target labels to select the NLL-minimizing single shift. It is a label-using reference, not a method, and not an upper bound for every displayed metric.}
\label{tab:ablation}
\small
\begin{tabular}{lccccc}
\toprule
Origin estimator & Acc$\uparrow$ & BAcc$\uparrow$ & F$_1$$\uparrow$ & AUROC$\uparrow$ & ECE$\downarrow$ \\
\midrule
Source & 0.625 & 0.649 & 0.588 & 0.733 & 0.149 \\
\textbf{PLOC (mean)} & 0.660 & 0.670 & 0.692 & 0.733 & 0.047 \\
EMA, $\alpha{=}0.01$ & 0.659 & 0.670 & 0.691 & 0.732 & 0.050 \\
EMA, $\alpha{=}0.2$ & 0.654 & 0.660 & 0.679 & 0.718 & 0.058 \\
Window, $w{=}100$ & 0.660 & 0.669 & 0.692 & 0.731 & 0.048 \\
Running $z$-score & 0.660 & 0.670 & 0.692 & 0.733 & 0.078 \\
Huber-clip, $\tau{=}2.5$ & 0.665 & 0.670 & 0.682 & 0.733 & 0.033 \\
\midrule
Oracle shift (uses labels) & 0.681 & 0.658 & 0.701 & 0.733 & 0.027 \\
\bottomrule
\end{tabular}
\end{table}

\FloatBarrier

\section{A generalization probe beyond tabular data}
\label{sec:vision}
PLOC is motivated by and evaluated on tabular streams, but its mechanism, recentering the output-logit origin, is modality-agnostic. As an illustrative probe we first applied it to a single-seed two-class CIFAR-10 corruption stream scored one image at a time by a small frozen convolutional network (Table~\ref{tab:vision}). The pattern is the same as in the tabular setting, and more pronounced. At batch size one the entropy methods either leave the source predictions unchanged or collapse: Tent falls to chance ($0.502$ accuracy) and EATA and SAR degenerate to a single-class predictor (F$_1=0$), whereas PLOC raises accuracy from $0.765$ to $0.831$ and cuts ECE from $0.125$ to $0.048$ while leaving AUROC unchanged at $0.902$. We report this only as evidence that the singleton pathology and the origin-centering fix are not specific to the tabular experiments. It is a probe, not a contribution, and a full vision study is out of scope here.

\begin{table}[!htbp]
\centering
\caption{Single-seed generalization probe: two-class CIFAR-10 corruption stream at $B{=}1$, small frozen CNN. Illustrative only.}
\label{tab:vision}
\small
\begin{tabular}{lccccc}
\toprule
Method & Acc & BAcc & F$_1$ & AUROC & ECE \\
\midrule
Source & 0.765 & 0.765 & 0.733 & 0.902 & 0.125 \\
Tent & 0.502 & 0.502 & 0.239 & 0.504 & 0.298 \\
EATA & 0.500 & 0.500 & 0.000 & 0.660 & 0.177 \\
SAR & 0.500 & 0.500 & 0.000 & 0.661 & 0.177 \\
LAME & 0.764 & 0.764 & 0.734 & 0.900 & 0.126 \\
\textbf{PLOC} & \textbf{0.831} & \textbf{0.831} & \textbf{0.831} & 0.902 & \textbf{0.048} \\
Deferred PLOC & 0.831 & 0.831 & 0.831 & 0.902 & 0.048 \\
PLOC-prob & 0.810 & 0.810 & 0.821 & 0.895 & 0.054 \\
\bottomrule
\end{tabular}
\end{table}

To probe the general $K$-class form of Section~\ref{sec:method}, we repeat the experiment on the full ten-class CIFAR-10 task over three independently trained checkpoints: a small frozen CNN scores one image at a time, and PLOC-K centers the ten logits coordinate-wise by their prequential running mean. We evaluate each checkpoint on thirteen singleton streams. The first is the unmodified CIFAR-10 test set, which serves as a control where no adaptation should be needed. The next eight apply a fixed image corruption to every test image: four corruption types (Gaussian noise, brightness, contrast, and blur), each at a moderate and a severe level (severities $3$ and $5$), with the class proportions kept balanced. The remaining four probe label shift: the test pool (clean, or with severity-$3$ Gaussian noise) is resampled to a long-tailed class distribution in which the most frequent class outnumbers the rarest by a factor $\rho$, for $\rho=10$ and $\rho=100$. Together, this gives $1+8+4=13$ streams per checkpoint. We observe that the results (Table~\ref{tab:vision10}) show a similar pattern as in the tabular datasets. Tent, EATA, and SAR again collapse to the $0.10$ chance level on every stream, including the control (accuracy between $0.087$ and $0.114$ over all thirty-nine stream--seed runs), and LAME returns the source predictions exactly unchanged on every stream, since its affinity graph is again a single node at $B{=}1$. PLOC-K leaves the control essentially untouched (accuracy $0.836\to0.834$, the self-modulation of Section~\ref{sec:theory}) while recovering large threshold and calibration losses on corrupted streams (e.g., accuracy $0.418\to0.581$ and ECE $0.212\to0.027$ under severity-5 contrast). Under pure label shift on clean images, it trades raw accuracy for a balanced operating point, consistent with Section~\ref{sec:scope}. Both predictions of Remark~\ref{rem:multiclass} are confirmed: deferred one-vs-one AUROC is unchanged to numerical precision on all thirty-nine stream--seed runs, while one-vs-rest AUROC shifts by up to $+0.07$, predominantly improving on shifted streams.

\begin{table}[!htbp]
\centering
\caption{Ten-class CIFAR-10 probe at $B{=}1$, averaged over three checkpoints, on six representative streams. The first two columns report the accuracy of the compared methods: Tent, EATA, and SAR coincide to three decimals on every stream (one column covers all three), and LAME returns predictions identical to the source. The middle columns show the value for the source model followed by the value after applying PLOC-K. The last two columns test Remark~\ref{rem:multiclass} for deferred PLOC-K: $\Delta$OvO is its change in the pairwise-logit one-vs-one macro AUROC relative to the source, which the remark predicts to be exactly zero (the largest observed $|\Delta\mathrm{OvO}|$ over all runs is $1.1\times10^{-8}$, numerical noise), while $\Delta$OvR is its change in one-vs-rest macro AUROC, which is not protected and indeed shifts, mostly upward, on the shifted streams.}
\label{tab:vision10}
\footnotesize
\setlength{\tabcolsep}{3pt}
\begin{tabular}{lccccccc}
\toprule
 & \multicolumn{2}{c}{Baselines (Acc$\uparrow$)} & \multicolumn{3}{c}{Source $\to$ PLOC-K} & \multicolumn{2}{c}{Deferred} \\
\cmidrule(lr){2-3}\cmidrule(lr){4-6}\cmidrule(lr){7-8}
Stream & Tent/EATA/SAR & LAME & Acc$\uparrow$ & ECE$\downarrow$ & NLL$\downarrow$ & $\Delta$OvO & $\Delta$OvR \\
\midrule
clean (control) & $0.100$ & $0.836$ & $0.836\!\to\!0.834$ & $0.034\!\to\!0.038$ & $0.501\!\to\!0.514$ & $0$ & $-0.000$ \\
gaussian noise s3 & $0.100$ & $0.175$ & $0.175\!\to\!0.332$ & $0.517\!\to\!0.076$ & $3.681\!\to\!1.860$ & $0$ & $+0.043$ \\
contrast s5 & $0.100$ & $0.418$ & $0.418\!\to\!0.581$ & $0.212\!\to\!0.027$ & $2.008\!\to\!1.216$ & $0$ & $+0.020$ \\
blur s3 & $0.100$ & $0.254$ & $0.254\!\to\!0.401$ & $0.368\!\to\!0.052$ & $3.053\!\to\!1.676$ & $0$ & $+0.027$ \\
shift $\rho{=}100$, clean & $0.087$ & $0.821$ & $0.821\!\to\!0.662$ & $0.046\!\to\!0.141$ & $0.526\!\to\!1.119$ & $0$ & $-0.009$ \\
shift $\rho{=}100$, noise s3 & $0.087$ & $0.085$ & $0.085\!\to\!0.224$ & $0.607\!\to\!0.181$ & $4.576\!\to\!2.114$ & $0$ & $+0.049$ \\
\bottomrule
\end{tabular}
\end{table}

\FloatBarrier

\section{Discussion}
\paragraph{When to use PLOC} PLOC is appropriate when a frozen tabular classifier must be deployed on a drifted stream one sample at a time, when target labels are unavailable for adaptation, and when the failure mode is a miscentered operating point rather than a fundamental loss of discriminative power. These conditions are common in the production settings that motivate the work: a credit model carried into a new economic period, a readmission model moved between hospitals, a scoring model facing a new cohort. In each case, the model still ranks cases sensibly, with its AUROC surviving the shift, but the decision threshold that was right for the source data now sits in the wrong place for the target data, and a single scalar correction recovers most of the lost accuracy and calibration without changing the ranking in the deferred variant.

\paragraph{Relationship to calibration and label shift} PLOC can be read as the minimal member of a family. Temperature scaling rescales logits by a learned positive factor. PLOC instead \emph{shifts} the logit origin by an unlabeled running mean. Logit adjustment and classical label-shift correction reweight by class frequencies that must be known or estimated from a batch. PLOC estimates nothing about the prior and needs no batch. The deferred variant's exact AUROC preservation places it on the conservative side of the adaptation tradeoff for ranking: it can only move the threshold, never reorder the cases. The cost of this safety is the scope limit of Section~\ref{sec:scope}: when the target prior genuinely shifts, and a batch is available to estimate it, a dedicated label-shift estimator can do more than an origin shift.

\paragraph{Threats to validity} Our claims are scoped to tabular classification under the TableShift splits, three backbones, and five checkpoints. The PFT3A comparison rests on a code-level analysis of its singleton degeneracies together with a verified reproduction of its batched numbers. We report the results of a guarded version of PFT3A only as a runnable reference and label it as such. 
The multi-class form is validated only by the controlled probe of Section~\ref{sec:vision}. Our study does not cover additional source families such as tree ensembles or non-stationary streams.

\FloatBarrier
\section{Conclusion}
\label{sec:conclusion}
Singleton tabular FTTA is a distinct identifiability problem, not just a small-batch case of ordinary FTTA. At batch size one, the batch-estimated priors, neighborhoods, and adaptation losses that drive existing methods are weakly identified or undefined, and the standard entropy- and normalization-based test-time adaptation methods leave the source predictions unchanged or nearly unchanged. We showed that adapting only the output-logit origin of a frozen classifier (one scalar, no labels, no model update) is a simple and stable alternative: PLOC improves accuracy, balanced accuracy, F$_1$, calibration, and likelihood over FTAT while leaving source AUROC essentially unchanged across five datasets, three backbones, and five checkpoints, and deferred PLOC carries an exact AUROC-preservation guarantee. A batch-size sweep confirms that the batch-based FTTA methods recover only once a larger batch is available, while PLOC is stable throughout. We see PLOC as the singleton baseline that future tabular test-time adaptation methods should compare against. 

\section*{Declarations}
\noindent\textbf{Declaration of competing interest.} The authors declare that they have no known competing financial interests or personal relationships that could have appeared to influence the work reported in this paper.

\noindent\textbf{Data availability.} All datasets are public TableShift~\cite{tableshift} benchmarks under their out-of-distribution splits. Code and aggregated per-cell metrics will be released upon publication.

\noindent\textbf{CRediT author statement.} \textbf{Mayank Sharma:} Methodology, Investigation, Writing -- original draft. \textbf{Rohit Kumar Mourya:} Methodology, Investigation, Writing -- original draft. \textbf{Pratik Mazumder:} Conceptualization, Writing -- review \& editing, Supervision, Project administration.

\FloatBarrier


\begin{thebibliography}{99}\small
\bibitem{ftat} Z.~Zhou, K.-Y.~Yu, L.-Z.~Guo, and Y.-F.~Li. Fully Test-time Adaptation for Tabular Data. \emph{Proceedings of the AAAI Conference on Artificial Intelligence}, 39(21):23027--23035, 2025. doi:10.1609/aaai.v39i21.34466.
\bibitem{pft3a} R.~He and J.~Shi. Prior-free Tabular Test-time Adaptation. \emph{ICLR}, 2026.
\bibitem{tent} D.~Wang, E.~Shelhamer, S.~Liu, B.~Olshausen, and T.~Darrell. Tent: Fully Test-Time Adaptation by Entropy Minimization. \emph{ICLR}, 2021.
\bibitem{eata} S.~Niu, J.~Wu, Y.~Zhang, Y.~Chen, S.~Zheng, P.~Zhao, and M.~Tan. Efficient Test-Time Model Adaptation without Forgetting. \emph{Proceedings of the 39th International Conference on Machine Learning}, PMLR 162:16888--16905, 2022.
\bibitem{sar} S.~Niu, J.~Wu, Y.~Zhang, Z.~Wen, Y.~Chen, P.~Zhao, and M.~Tan. Towards Stable Test-Time Adaptation in Dynamic Wild World. \emph{ICLR}, 2023.
\bibitem{lame} M.~Boudiaf, R.~Mueller, I.~Ben Ayed, and L.~Bertinetto. Parameter-Free Online Test-Time Adaptation. \emph{CVPR}, pages 8344--8353, 2022.
\bibitem{cotta} Q.~Wang, O.~Fink, L.~Van Gool, and D.~Dai. Continual Test-Time Domain Adaptation. \emph{CVPR}, pages 7201--7211, 2022.
\bibitem{note} T.~Gong, J.~Jeong, T.~Kim, Y.~Kim, J.~Shin, and S.-J.~Lee. NOTE: Robust Continual Test-time Adaptation Against Temporal Correlation. \emph{NeurIPS}, 2022.
\bibitem{rotta} L.~Yuan, B.~Xie, and S.~Li. Robust Test-Time Adaptation in Dynamic Scenarios. \emph{CVPR}, pages 15922--15932, 2023.
\bibitem{guo} C.~Guo, G.~Pleiss, Y.~Sun, and K.~Q.~Weinberger. On Calibration of Modern Neural Networks. \emph{Proceedings of the 34th International Conference on Machine Learning}, PMLR 70:1321--1330, 2017.
\bibitem{saerens} M.~Saerens, P.~Latinne, and C.~Decaestecker. Adjusting the Outputs of a Classifier to New a Priori Probabilities: A Simple Procedure. \emph{Neural Computation}, 14(1):21--41, 2002. doi:10.1162/089976602753284446.
\bibitem{bbse} Z.~Lipton, Y.-X.~Wang, and A.~Smola. Detecting and Correcting for Label Shift with Black Box Predictors. \emph{Proceedings of the 35th International Conference on Machine Learning}, PMLR 80:3122--3130, 2018.
\bibitem{alexandari} A.~Alexandari, A.~Kundaje, and A.~Shrikumar. Maximum Likelihood with Bias-Corrected Calibration is Hard-To-Beat at Label Shift Adaptation. \emph{Proceedings of the 37th International Conference on Machine Learning}, PMLR 119:222--232, 2020.
\bibitem{rlls} K.~Azizzadenesheli, A.~Liu, F.~Yang, and A.~Anandkumar. Regularized Learning for Domain Adaptation under Label Shifts. \emph{ICLR}, 2019.
\bibitem{menon} A.~K.~Menon, S.~Jayasumana, A.~S.~Rawat, H.~Jain, A.~Veit, and S.~Kumar. Long-tail Learning via Logit Adjustment. \emph{ICLR}, 2021.
\bibitem{t3a} Y.~Iwasawa and Y.~Matsuo. Test-Time Classifier Adjustment Module for Model-Agnostic Domain Generalization. \emph{NeurIPS}, 2021.
\bibitem{dua} M.~J.~Mirza, J.~Micorek, H.~Possegger, and H.~Bischof. The Norm Must Go On: Dynamic Unsupervised Domain Adaptation by Normalization. \emph{CVPR}, pages 14765--14775, 2022.
\bibitem{fttransformer} Y.~Gorishniy, I.~Rubachev, V.~Khrulkov, and A.~Babenko. Revisiting Deep Learning Models for Tabular Data. \emph{NeurIPS}, 2021.
\bibitem{tabtransformer} X.~Huang, A.~Khetan, M.~Cvitkovic, and Z.~Karnin. TabTransformer: Tabular Data Modeling Using Contextual Embeddings. \emph{arXiv:2012.06678}, 2020.
\bibitem{gama} J.~Gama, R.~Sebasti\~ao, and P.~P.~Rodrigues. On Evaluating Stream Learning Algorithms. \emph{Machine Learning}, 90(3):317--346, 2013. doi:10.1007/s10994-012-5320-9.
\bibitem{tableshift} J.~Gardner, Z.~Popovic, and L.~Schmidt. Benchmarking Distribution Shift in Tabular Data with TableShift. \emph{NeurIPS}, 2023.
\bibitem{heloc} FICO. The Explainable Machine Learning Challenge. FICO Community, 2018. https://community.fico.com/s/explainable-machine-learning-challenge.
\bibitem{strack} B.~Strack, J.~P.~DeShazo, C.~Gennings, J.~L.~Olmo, S.~Ventura, K.~J.~Cios, and J.~N.~Clore. Impact of HbA1c Measurement on Hospital Readmission Rates: Analysis of 70,000 Clinical Database Patient Records. \emph{BioMed Research International}, 2014:781670, 2014. doi:10.1155/2014/781670.
\bibitem{assistments} N.~T.~Heffernan and C.~L.~Heffernan. The ASSISTments Ecosystem: Building a Platform that Brings Scientists and Teachers Together for Minimally Invasive Research on Human Learning and Teaching. \emph{International Journal of Artificial Intelligence in Education}, 24(4):470--497, 2014. doi:10.1007/s40593-014-0024-x.
\bibitem{anes} American National Election Studies. ANES Time Series Cumulative Data File [dataset and documentation]. September 16, 2022 version. University of Michigan and Stanford University, 2022.
\bibitem{folktables} F.~Ding, M.~Hardt, J.~Miller, and L.~Schmidt. Retiring Adult: New Datasets for Fair Machine Learning. \emph{NeurIPS}, 2021.
\bibitem{asoc_covshift} H.~J.~Song and S.~B.~Park. An Adapted Surrogate Kernel for Classification under Covariate Shift. \emph{Applied Soft Computing}, 69:435--442, 2018. doi:10.1016/j.asoc.2018.04.060.
\bibitem{asoc_transfer} T.~Han, C.~Liu, R.~Wu, and D.~Jiang. Deep Transfer Learning with Limited Data for Machinery Fault Diagnosis. \emph{Applied Soft Computing}, 103:107150, 2021. doi:10.1016/j.asoc.2021.107150.
\bibitem{asoc_dann} J.~Kim and J.~Lee. Instance-Based Transfer Learning Method via Modified Domain-Adversarial Neural Network with Influence Function: Applications to Design Metamodeling and Fault Diagnosis. \emph{Applied Soft Computing}, 123:108934, 2022. doi:10.1016/j.asoc.2022.108934.
\bibitem{asoc_m2m} S.~Wang, M.~Gao, H.~Wu, F.~Luo, F.~Jiang, and L.~Tao. Many-to-Many: Domain Adaptation for Water Quality Prediction. \emph{Applied Soft Computing}, 167:112381, 2024. doi:10.1016/j.asoc.2024.112381.
\bibitem{asoc_stream_ens} K.~K.~Wankhade, K.~C.~Jondhale, and S.~S.~Dongre. A Clustering and Ensemble Based Classifier for Data Stream Classification. \emph{Applied Soft Computing}, 102:107076, 2021. doi:10.1016/j.asoc.2020.107076.
\bibitem{asoc_online_inc} S.~Zhang, J.~Liu, and X.~Zuo. Adaptive Online Incremental Learning for Evolving Data Streams. \emph{Applied Soft Computing}, 105:107255, 2021. doi:10.1016/j.asoc.2021.107255.
\bibitem{asoc_imb_stream} J.~Klikowski and M.~Wo\'zniak. Deterministic Sampling Classifier with Weighted Bagging for Drifted Imbalanced Data Stream Classification. \emph{Applied Soft Computing}, 122:108855, 2022. doi:10.1016/j.asoc.2022.108855.
\bibitem{asoc_active_stream} Y.~Guo, J.~Pu, B.~Jiao, Y.~Peng, D.~Wang, and S.~Yang. Online Semi-Supervised Active Learning Ensemble Classification for Evolving Imbalanced Data Streams. \emph{Applied Soft Computing}, 155:111452, 2024. doi:10.1016/j.asoc.2024.111452.
\bibitem{asoc_credit} C.-F.~Wu, S.-C.~Huang, C.-C.~Chiou, and Y.-M.~Wang. A Predictive Intelligence System of Credit Scoring Based on Deep Multiple Kernel Learning. \emph{Applied Soft Computing}, 111:107668, 2021. doi:10.1016/j.asoc.2021.107668.
\bibitem{asoc_uq_calib} X.~Liu, L.~Zhang, W.~Huang, D.~Cheng, H.~Wu, and A.~Song. Fixing Deep Early Exit Ensembles for Sensor-Based Human Activity Recognition through Uncertainty Quantification. \emph{Applied Soft Computing}, 184:113861, 2025. doi:10.1016/j.asoc.2025.113861.
\end{thebibliography}
\end{document}